\documentclass[11pt]{article} 
\usepackage{newtxtext}
\usepackage[left=1.25in, top=1in, bottom=1in, right=1.25in]{geometry}

\usepackage{graphicx}
\usepackage{float}
\usepackage[utf8]{inputenc} 
\usepackage[T1]{fontenc}    
\usepackage{url}            
\usepackage{booktabs}       
\usepackage{amsfonts}       
\usepackage{nicefrac}       
\usepackage{xcolor}         

\usepackage[colorlinks=true,linkcolor=magenta,citecolor=blue]{hyperref}%

\newcommand\numberthis{\addtocounter{equation}
{1}\tag{\theequation}}

\newcommand{\tr}{\textup{tr}}
\hypersetup{backref=page}
\usepackage{url}            
\usepackage{nicefrac}       
\usepackage{xcolor}  
\usepackage{amsfonts, amsmath, amssymb, bm, mathtools}
\usepackage{cleveref}
\usepackage{dsfont}
\usepackage{cancel}
\usepackage{enumitem}
\usepackage{natbib}
\usepackage{thmtools} 
\usepackage{thm-restate}
\usepackage{caption}
\usepackage{nicefrac}
\usepackage{xfrac}

\usepackage{amsmath,amsfonts,bm}

\def\1{\bm{1}}

 \def\mW{{\bm{W}}}

\DeclareMathAlphabet{\mathsfit}{\encodingdefault}{\sfdefault}{m}{sl}
\SetMathAlphabet{\mathsfit}{bold}{\encodingdefault}{\sfdefault}{bx}{n}

\def\gG{{\mathcal{G}}}
\def\gH{{\mathcal{H}}}

\def\gL{{\mathcal{L}}}

\def\gN{{\mathcal{N}}}

\def\gR{{\mathcal{R}}}

\def\sR{{\mathbb{R}}}

\DeclareMathOperator*{\argmin}{arg\,min}

\DeclareMathOperator{\Exp}{\mathop \mathbb{E}}

\newcommand{\ind}{\mathds{1}}  

\newcommand{\norm}[1]{\left\|#1\right\|}
\newcommand{\snorm}[2]{\left\|#1\right\|_{S_{#2}}}

\newcommand{\inner}[2]{\left\langle #1,#2 \right\rangle}

\newcommand{\diff}{\mathrm{d}}

\newcommand*{\E}{\E}

\newcommand{\Lbar}{\bar{\mathcal L}}

\newcommand{\zhiyuan}[1]{{\color{red} [ZL: #1]}}

\usepackage{amsthm}
\newtheorem{definition}{Definition}
\newtheorem{theorem}{Theorem}
\newtheorem{lemma}{Lemma}

\newtheorem{example}{Example}
\newtheorem{corollary}{Corollary}

\title{The Inductive Bias of Flatness Regularization \\ for Deep Matrix Factorization}

\author{%
    Khashayar Gatmiry \\ 
    MIT\\
    \texttt{gatmiry@mit.edu }
    \and  
     Zhiyuan Li \\
    Stanford University \\
    \texttt{zhiyuanli@stanford.edu}
\and
     Ching-Yao Chuang \\
     MIT \\
     \texttt{cychuang@mit.edu}
     \and 
     Sashank Reddi \\
     Google \\
     \texttt{sashank@google.com}
     \and 
     Tengyu Ma \\
     Stanford University\\
     \texttt{tengyuma@stanford.edu}
      \and 
      Stefanie Jegelka \\
      MIT \\
      \texttt{stefje@csail.mit.edu}
}
\date{}

\begin{document}

\maketitle

\begin{abstract}
Recent works on over-parameterized neural networks have shown that  the stochasticity in optimizers has the implicit regularization effect of minimizing the sharpness of the loss function (in particular, the trace of its Hessian) over the family zero-loss solutions. More explicit forms of flatness regularization also empirically improve the generalization performance. However, it remains unclear why and when flatness regularization leads to better generalization. 
This work takes the first step toward understanding the inductive bias of the minimum trace of the Hessian solutions in an important setting: learning deep linear networks from linear measurements, also known as \emph{deep matrix factorization}. We show that for all depth greater than one, with the standard Restricted Isometry Property (RIP) on the measurements, minimizing the trace of Hessian is approximately equivalent to minimizing the Schatten 1-norm of the corresponding end-to-end matrix parameters (i.e., the product of all layer matrices), which in turn leads to better generalization. We empirically verify our theoretical findings on synthetic datasets.

\end{abstract}


\section{Introduction}
Modern deep neural networks are typically over-parametrized and equipped with huge model capacity, but surprisingly, they generalize well when trained using stochastic gradient descent (SGD) or its variants~\citep{zhang2017understanding}. A recent line of research suggested the \emph{implicit bias} of SGD as a possible explanation to this mysterious ability.
In particular, \citet{damian2021label,li2021happens,arora2022understanding,lyu2022understanding,wen2022does,Liu2022SamePL} have shown that SGD can implicitly minimize the \emph{sharpness} of the training loss, in particular, the trace of the Hessian of the training loss, to obtain the final model. However, despite the strong empirical evidence on the correlation between various notions of sharpness and generalization~\citep{keskar2016large,jastrzkebski2017three,neyshabur2017exploring,jiang2019fantastic} and the effectiveness of using sharpness regularization on improving generalization~\citep{foret2020sharpness,wu2020adversarial,zheng2021regularizing,norton2021diametrical}, the connection between penalization of the sharpness of training loss and better generalization still remains majorly unclear~\citep{dinh2017sharp,andriushchenko2023modern} and has only been proved in the context of two-layer linear models~\citep{li2021happens,nacson2022implicit,ding2022flat}. To further understand this connection beyond the two layer case, we study the inductive bias of penalizing the \emph{trace of the Hessian} of training loss and its effect on the \emph{generalization} in an important theoretical deep learning setting: \emph{deep linear networks} (or equivalently, \emph{deep matrix factorization}~\citep{arora2019implicit}). We start by briefly describing the problem setup.

\paragraph{Deep Matrix Factorization.} Consider an $L$-layer deep network where $L\in \mathbb{N}^+, L\ge 2$ is the depth of the model. Let $W_i\in\mathbb{R}^{d_i\times d_{i-1}}$ and $d_i$ denote the layer weight matrix and width of the $i^{\text{th}}$ ($i \in [L]$) layer respectively. We use $\mathbf{W}$ to denote the concatenation of all the parameters $(W_1,\ldots, W_L)$ and define the \emph{end-to-end matrix} of $\mathbf{W}$ as 
\begin{align}
    E(\mathbf{W}) \triangleq W_LW_{L-1}\cdots W_1.\label{eq:Edefinition}
\end{align}

In this paper, we focus on models that are linear in the space of the end-to-end matrix $E(W)$. Suppose $M^*\in \mathbb{R}^{d_L\times d_0}$ is the target end-to-end matrix, and we observe $n$ linear measurements (matrices) $A_i\in \mathbb{R}^{d_L\times d_0}$ and the corresponding labels $b_i = \inner{A_i}{M^*}$. 
The training loss of $\mathbf{W}$ is the mean-squared error (MSE) between the prediction $\inner{A_i}{W_LW_{L-1}\cdots W_1}$ and the observation $b_i$:
\begin{align}
    \gL(\mathbf{W})\triangleq \frac{1}{n}\sum_{i=1}^n\left(\inner{A_i}{W_LW_{L-1}\cdots W_1}-b_i\right)^2.\label{eq:lmp}
\end{align}
Throughout this paper, we assume that  $d_i\ge \min (d_0,d_L)$ for each $i\in[L]$ and, thus, the image of the function $E(\cdot)$ is the entire $\mathbb{R}^{d_L\times d_0}$. In particular, this ensures that the deep models are sufficiently expressive in the sense that $\min\limits_{\mW}\gL(\mW) = 0$. For this setting, we aim to understand the structure of the trace of the Hessian minimization, as described below. The trace of Hessian is the sum of the eigenvalues of Hessian, which is an indicator of sharpness and it is known that variants of SGD, such as label noise SGD or 1-SAM, are biased toward models with a smaller trace of Hessian~\citep{li2021happens,wen2022does}. 

{\bf Min Trace of Hessian Interpolating Solution.} Our primary object of study is the interpolating solution with the minimum trace of Hessian, defined as:  
\begin{align}\label{eq:min_trace_of_hessian_interpolating_soln}
    \mW^* \in \argmin_{\mW: \gL(\mW)=0} \tr[\nabla^2\gL(\mW)]. 
\end{align}
As we shall see shortly, the solution to the above optimization problem is not unique. We are interested in understanding the underlying structure of any minimizer $\mW^*$. This will, in turn, inform us about the generalization nature of these solutions.


\subsection{Main Results}\label{sec:our_results}

Before delving into the technical details, we state our main results in this section. This also serves the purpose of highlighting the primary technical contributions of the paper.
First, since the generalization of $\mW$ only depends on its end-to-end matrix $E(\mW)$, it is informative to derive the properties of $E(\mW^*)$ for any min trace of the Hessian interpolating solution $\mW^*$ defined in \eqref{eq:min_trace_of_hessian_interpolating_soln}. Indeed, penalizing the trace of Hessian in the $W$ space induces an equivalent penalization in the space of the end-to-end parameters. More concretely, given an end-to-end parameter $M$, let the induced regularizer $F(M)$ denote the trace of Hessian of the training loss at $\mW$ among all $\mW$'s that instantiate the end-to-end matrix $M$ i.e., $E(\mW) = M$. 
\begin{definition}[Induced Regularizer]\label[definition]{defi:induced_regularizer} Suppose $M\in\mathbb{R}^{d_L\times d_0}$ is an end-to-end parameter that fits the training data perfectly (that is, $\inner{A_i}{M}=b_i,\ \forall i\in[n]$). We define the \emph{induced regularizer} as 
    \begin{align}
        F(M) \triangleq \min_{\mW:E(\mW)=M} \tr[\nabla^2 \gL(\mW)]\label{eq:Fdefinition1} 
    \end{align}
\end{definition}
Since the image of $E(\cdot)$ is the entire $\mathbb{R}^{d_L\times d_{0}}$ by our assumption that $d_i \geq \min(d_0, d_L)$,  function $F$ is well-defined for all $M \in \mathbb{R}^{d_L \times d_0}$. It is easy to see that minimizing the trace of the Hessian in the original  parameter space (see \eqref{eq:min_trace_of_hessian_interpolating_soln}) is equivalent to penalizing $F(M)$ in the end-to-end parameter. Indeed, the minimizers of the implicit regularizer in the end-to-end space are related to the minimizers of the implicit regularizer in the $\mW$ space, i.e., 
\begin{align*}
\argmin\limits_{M: \gL'(M)=0} F(M) =\left\{ E(\mW^*)\mid \mW^* \in \argmin\limits_{\mW: \gL(\mW)=0} \tr[\nabla^2\gL(\mW)] \right\}, 
\end{align*}
where for any $M\in \mathbb{R}^{d_L\times d_0}$, we define $\gL'(M) \triangleq \frac{1}{n}\sum_{i=1}\left(\inner{A_i}{M}-b_i\right)^2$ and thus $\gL(\mW) = \gL'(E(\mW))$. This directly follows from the definition of $F$ in \eqref{eq:Fdefinition1}. 
Our main result characterizes the induced regularizer $F(M)$ when the data satisfies the RIP property. 

\begin{theorem}[Induced regularizer under RIP]\label{thm:RIPbias}
   Suppose the linear measurements $\{A_i\}_{i=1}^n$ satisfy the $(1,\delta)$-RIP condition. 
   \begin{enumerate}
   	\item For any $M\in\mathbb{R}^{d_L\times d_0}$ such that $\inner{A_i}{M}=b_i,\ \forall i\in[n]$, it holds that 
    \begin{align}
       (1-\delta) L{(d_0 d_L)}^{1/L} \|M\|_*^{2(L-1)/L} \le F(M)\le (1+\delta) L{(d_0 d_L)}^{1/L} \|M\|_*^{2(L-1)/L}.
    \end{align}
   	\item     Let $\mW^* \in \argmin_{\mW: \gL(\mW)=0} \tr[\nabla^2\gL(\mW)]$ be an interpolating solution with minimal trace of Hessian . Then $E(\mW^*)$  roughly minimizes the nuclear norm among all interpolating solutions of $\gL'$. That is,
    \begin{align*}
        \| E(\mW^*)\|_* \leq \frac{1 + \delta}{1-\delta} \min_{ \gL'(M)=0}\|M\|_*.
    \end{align*}
   \end{enumerate}
\end{theorem}

\begin{table}[t]
    \centering
    {\renewcommand{\arraystretch}{1.7}
    \begin{tabular}{l|l|l}
        Settings  &  Induced Regularizer $F(M)/L$ & Theorem\\
        \hline
         $(1,\delta)$-RIP & $(1\pm O(\delta)){(d_0 d_L)}^{1/L} \|M\|_*^{2-\nicefrac{2}{L}}$ & \Cref{thm:RIPbias} \\
        $L=2$ &  $\norm{\left(\frac{1}{n}A_iA_i^\top \right)^{\nicefrac{1}{2} } M \left(\frac{1}{n}A_i^\top A_i \right)^{\nicefrac{1}{2}} }_*$ & \Cref{thm:induced_regularizer_depth_2}~(\citep{ding2022flat})\\
        
        $n=1$ &  $\snorm{\Big(A^T M\Big)^{L-1}A^T}{2/L}^{2/L}$ & \Cref{thm:induced_regularizer_single_measurement}\\
        \hline
    \end{tabular}}
    \vspace{0.1cm}
    \caption{\small Summary of properties of the induced regularizer in the end-to-end matrix space. Here $\snorm{\cdot}{p}$ denotes the Schatten $p$-norm for $p\in[1,\infty]$ and Schatten $p$-quasinorm for $p\in(0,1)$ (see \Cref{defi:schattern_p}). $\norm{\cdot}_*$ denotes the Schatten 1-norm, also known as the nuclear norm. }    \label{tab:summary_closed_form_induced_regularizer}
\end{table}

However, for more general cases, it is challenging to compute the closed-form expression of $F$. In this work, we derive  closed-form expressions for $F$ in the following two cases: (1) depth $L$ is equal to $2$ and (2) there is only one measurement, \emph{i.e.}, $n=1$ (see \Cref{tab:summary_closed_form_induced_regularizer}).
Leveraging the above characterization of induced regularzier, we obtain the following result on the generalization bounds:

\begin{theorem}[Recovery of the ground truth under RIP]\label[theorem]{thm:ripgeneralization}
 Suppose the linear measurements $\{(A_i)\}_{i=1}^n$ satisfy the $(2,\delta(n))$-RIP~(\Cref{defi:rip}). Then for any $\mW^* \in \argmin\limits_{\mW: \gL(\mW)=0} \tr[\nabla^2\gL(\mW)]$, we have 
\vspace{-0.2cm}
\begin{align}
    \| E(\mW^*) - M^*\|_F^2 \leq \frac{8\delta(n)}{(1-\delta(n))^2}\|M^*\|_*^2.\label{eq:ripgeneralization}
\end{align}
where $\delta(n)$ depends on the number of measurements $n$ and the distribution of the measurements. 
\end{theorem}

If we further suppose $\{A_i\}_{i=1}^n$ are independently sampled from some distribution over $\mathbb{R}^{d_L\times d_0}$ satisfying that $\mathbb{E}_A \inner{A}{M}^2 = \norm{M}_F^2$, \emph{e.g.},  the standard multivariate Gaussian distribution, denoted by $\gG_{d_L \times d_0}$, we know $\delta(n) = O(\sqrt{\tfrac{d_L+d_0}{n}})$ from \citet{candes2011tight} 
(see \Cref{sec:rip_preliminary} for more examples). 

\begin{theorem}\label[theorem]{thm:generalizationfastrate}
     For $n \geq \Omega(r(d_0+d_L))$, with probability at least $1-\exp(\Omega(d_0+d_L))$ over the randomly sampled $\{A_i\}_{i=1}^n$ from multivariate Gaussian distribution $\gG$, for any minimum trace of Hessian interpolating solution $\mW^* \in \argmin\limits_{\mW: \gL(\mW)=0} \tr[\nabla^2\gL(\mW)]$, the population loss $\overline \gL(\mW^*) \triangleq \mathbb E_{A \sim \gG}(\inner{ A}{ E(\mW^*) } - \inner{ A}{ M^* }) ^2$ satisfies that 
    \begin{align*}
        \overline \gL(\mW^*)  = \norm{E(\mW^*) - M^*}_F^2\le   O\Big(\frac{d_0+d_L}{n}\|M^*\|_*^2\log^3 n\Big).
    \end{align*}
\end{theorem}
Next, we state a lower bound for the conventional estimator for overparameterized models that minimizes the norm. The lower bound states that, to achieve a small error, the number of samples should be as large as the product of the dimensions of the end-to-end matrix $d_0d_L$ as opposed to $d_0 + d_L$ in case of the min trace of Hessian minimizer. It is proved in Appendix~\ref{sec:lowerboundproof}.
\begin{theorem}[Lower bound for $\ell_2$ regression]\label{thm:lowerbound}
    Suppose $\{A_i\}_{i=1}^n$ are randomly sampled from multivariate Gaussian distribution $\gG$, let $ \tilde \mW = \argmin_{\mW: \gL(\mW)=0} \|E(\mW)\|_F $ to be the minimum Frobenius norm interpolating solution, then the expected population loss is 
    \begin{align*}
        \Exp\overline{\gL}(\tilde \mW) = (1-\tfrac{\min\{n, d_0d_L\}}{d_0d_L}) \norm{M^*}_F^2.
    \end{align*}
\end{theorem}
The lower bound in Theorem~\ref{thm:lowerbound} shows in order to obtain an $O(1)$-relatively accurate estimates of the ground truth in expectation, namely to guarantee $ \Exp\overline{\gL}(\tilde \mW) \leq O(1)\|M^*\|_F^2$, the minimum Frobenius norm interpolating solution needs at least $\Omega(d_0d_L)$ samples. In contrast, the minimizer of trace of Hessian in the same problem only requires $O((d_0+d_L)\|M^*\|_*^2/\|M^*\|_F^2)$ samples, which is at most $O(\frac{r(d_0 + d_L)}{d_0d_L})$ fraction of the number of samples that is required for the minimum Frobenius norm interpolator (Theorem~\ref{thm:generalizationfastrate}). 
\section{Related Work}\label{sec:related_work}

\paragraph{Connection Between Sharpness and Generalization.} Research on the connection between generalization and sharpness dates back to \citet{hochreiter1997flat}. \citet{keskar2016large} famously observe that when increasing the batch size of SGD, the test error and the sharpness of the learned solution both increase. \citet{jastrzkebski2017three} extend this observation and found that there is a positive correlation between sharpness and the ratio between learning rate and batch size. \citet{jiang2019fantastic} perform a large-scale empirical study on various notions of generalization measures and show that sharpness-based measures correlate with generalization best. \citet{Liu2022SamePL} find that among language models with the same validation pretraining loss, those that have smaller sharpness can have better downstream performance. On the other hand, \citet{dinh2017sharp} argue that for networks with scaling invariance, there always exist models with good generalization but with arbitrarily large sharpness. We note this does not contradict our main result here, which only asserts the interpolation solution with a minimal trace of Hessian generalizes well, but not vice versa. Empirically, sharpness minimization is also a popular and effective regularization method for overparametrized models~\citep{norton2021diametrical,foret2021sharpnessaware,zheng2021regularizing,wu2020adversarial,kwon2021asam,liu2022towards,zhuang2022surrogate,zhao2022penalizing,andriushchenko2022towards}.

\paragraph{Implicit Bias of Sharpness Minimization.}
Recent theoretical works \citep{blanc2019implicit,damian2021label,li2021happens,Liu2022SamePL} show that SGD with label noise is implicitly biased toward local minimizers with a smaller trace of Hessian under the assumption that the minimizers locally connect as a manifold. Such a manifold setting is empirically verified by \citet{draxler2018essentially,garipov2018loss} in the sense that the set of minimizers of the training loss is path-connected. It is the same situation for the deep matrix factorization problem studied in this paper, although we do not study the optimization trajectory. Instead, we directly study properties of the minimum trace of Hessian interpolation solution. 

Sharpness-reduction implicit bias can also happen for deterministic GD. \citet{arora2022understanding} show that normalized GD implicitly penalizes the largest eigenvalue of the Hessian. \citet{ma2022multiscale} argues that such sharpness reduction phenomena can also be caused by a multi-scale loss landscape. \citet{lyu2022understanding} show that GD with weight decay  on a scale-invariant loss function implicitly decreases the spherical sharpness, \emph{i.e.}, the largest eigenvalue of the Hessian evaluated at the normalized parameter. Another line of work focuses on the sharpness minimization effect of a large learning rate in GD, assuming that it converges at the end of training. This has been studied mainly through linear stability analysis~\citep{wu2018sgd,cohen2021gradient,ma2021linear,cohen2022adaptive}. Recent theoretical analysis~\citep{damian2022self,li2022analyzing} showed that the sharpness minimization effect of a large learning rate in GD does not necessarily rely on convergence and linear stability, through a four-phase characterization of the dynamics at the Edge of Stability regime~\citep{cohen2021gradient}.


\paragraph{Sharpness-related Generalization Bounds.} Most existing sharpness-related generalizations depend on not only the sharpness of the training loss but also other complexity measures like a 
norm of the parameters or even undesirable dependence on the number of parameters~\citep{dziugaite2017computing,wei2019data,wei2019improved,foret2021sharpnessaware,norton2021diametrical}.  In contrast, our result only involves the trace of Hessian but not parameter norm or the number of parameters, \emph{e.g.}, our result holds for any (large) width of intermediate layers, $d_1,\ldots,d_{L-1}$. 


\paragraph{Implicit Bias of Gradient Descent on Matrix Factorization.} At first glance, overfitting could happen when the number of linear measurements is less than the size of the groundtruth matrix. Surprisingly, a recent line of works~\citep{gunasekar2017implicit,arora2019implicit,gissin2019implicit,li2020towards,razin2020implicit,belabbas2020implicit,jacot2021deep,razin2021implicit} has shown that GD starting from small initialization has a good implicit bias towards solutions with approximate recovery of ground truth. Notably, \citet{gunasekar2017implicit} show that for depth $2$, GD from infinitesimal initialization is implicitly biased to the minimum nuclear norm solution under commuting measurements and \citet{arora2019implicit} generalize this results to deep matrix factorization for any depth. This is very similar to our main result that for all depth ($\ge 2$) the implicit regularization is minimizing nuclear norm, though the settings are different. Moreover, when the measurements satisfy RIP, \citet{li2017algorithmic,stoger2021small} show that GD exactly recovers the ground truth.

\paragraph{Provable Generalization of Flatness Regularization for Two-layer Models.}\ To our best knowledge, most existing generalization analysis for flat regularization are for two-layer models, \emph{e.g.}, \citet{li2021happens} shows that the min trace of hessian interpolating solution of 2-layer diagonal linear networks can recover sparse ground truth on gaussian or boolean data, and \citet{nacson2022implicit} proves a generalization bound for the interpolating solutions with the smallest maximum eigenvalue of Hessian for non-centered data.
\citet{ding2022flat} is probably the most related work to ours, which shows that the trace of Hessian implicit bias for two-layer matrix factorization is a rescaled version of the nuclear norm of the end-to-end matrix. Using this formula, they further prove that the flattest solution in this problem recovers the low-rank ground truth. However, matrix factorization with more than two layers is fundamentally more challenging compared to the depth two case; while we managed to obtain a formula for the trace of Hessian for deeper networks given a single measurement (see \Cref{thm:induced_regularizer_single_measurement}), as far as we know, one in general cannot obtain a closed-form solution for the trace of Hessian regularizer as a function of the end-to-end matrix for multiple measurements. In this work, we discover a way to bypass this hardness by showing that minimizing the trace of Hessian regularizer for a fixed end-to-end matrix approximately amounts to the nuclear norm of the end-to-end matrix, when the linear measurements satisfy the RIP property. As a cost of this approximation, we are not able to show the exact recovery of the low-rank ground truth, but only up to a certain precision.

\paragraph{Sharpness Minimization in Deep Diagonal Linear Network.}\  \citet{ding2022flat}  show that the minimizer of trace of Hessianin a deep diagonal matrix factorization model with Gaussian linear measurements becomes the Schatten $2-2/L$ norm of a rescaled version of the end to end matrix. At first glance, their result might seem contradictory to our result in the RIP setup, as their implicit regularization is not always the Nuclear norm --- the sparsity regularization vanishes when $L\to\infty$. Similar results have been obtained by \citet{nacson2022implicit} for minimizing a different notion of sharpness among all interpolating solutions, the largest eigenvalue of Hessian, on the same diagonal linear models. The subtle difference is that since we consider the more standard setting without assuming the weight matrices are all diagonal, then in the calculation of the trace of Hessian of the loss we need to also differentiate the loss with respect to the non-diagonal entries, even though their values are zero, which is quite different from $\ell_p$ norm regularization. 
This curiously shows the complicated interplay between the geometry of the loss landscape and the implicit bias of the algorithm.

\section{Preliminaries}


{\bf Notation.} We use $[n]$ to denote $\{1,2,\ldots, n\}$ for every $n\in\mathbb{N}$.  We use $\norm{M}_F$, $\norm{M}_*$, $\norm{M}_2$ and $\tr(M)$ to denote the Frobenius norm, nuclear norm, spectral norm and trace of matrix $M$ respectively. For any function $f$ defined over set $S$ such that $\min_{x\in S} f(x)$ exists, we use $\argmin_S f$ to denote the set $\{y\in S \mid f(y) = \min_{x\in S} f(x)\}$. Given a matrix $M$, we use $h_M$ to denote the linear map $A\mapsto \inner{A}{M}$. We use $\gH_r$ to to denote the set $\gH_r \triangleq \{h_M\mid \norm{M}_*\le r\}$. $M_{i:}$ and $M_{:j}$ are used to denote the $i$th row and $j$th column of the matrix $M$.

The following definitions will be important to the technical discussion in the paper. 

\paragraph{Rademacher Complexity.} Given $n$ data points $\{A_i\}_{i=1}^n$, the \emph{empirical Rademacher complexity} of function class $\gH$ is defined as 
\begin{align}
  \gR_n(\gH) &= \frac{1}{n}\Exp_{\epsilon \sim \{\pm 1\}^n} \sup_{h\in \gH}\sum_{i=1}^n  \epsilon_i h(A_i).\notag
\end{align}
Given a distribution $P$, the \emph{population Rademacher complexity} is defined as follows:
\(\overline \gR_n(\gH) = \Exp\limits_{A_i\overset{iid}{\sim} P} \gR_n(\gH)\). This is mainly used to upper bound the generalization gap of SGD.

\begin{definition}[Schatten $p$-(quasi)norm]\label[definition]{defi:schattern_p}
Given any $d,d'\in\mathbb{N}^+$, $p\in(0,\infty)$ a matrix $M\in \mathbb{R}^{d\times d'}$ with singular values $\sigma_1(M), \ldots, \sigma_{\min(d,d')}(M)$, we define the Schattern $p$-(semi)norm as
\begin{align*}
    \snorm{M}{p} = \left( \sum\nolimits_{i=1}^{\min(d,d')} {\sigma_i^p(M)}\right)^{1/p}.
\end{align*}
\end{definition}
Note that in this definition $\snorm{\cdot}{p}$ is a norm only when $p\ge 1$. When $p \in (0,1)$, the triangle inequality does not hold. Note that when $p \in (0,1)$, $\snorm{A+B}{p}\le 2^{\nicefrac{1}{p}-1}(\snorm{A}{p}+ \snorm{B}{p})$ for any matrices $A$ and $B$, however, $2^{\nicefrac{1}{p}-1} > 1$.

We use $L$ to denote the depth of the linear model and $\mW = (W_1,\ldots, W_L)$ to denote the parameters, where $W_i\in\mathbb{R}^{d_{i}\times d_{i-1}}$. We assume that $d_i\ge \min (d_0,d_L)$ for each $i\in[L-1]$ and, thus, the image of $E(\mW)$ is the entire $\mathbb{R}^{d_L\times d_0}$. 
Following is a simple relationship between nuclear norm and Frobenius norm that is used frequently in the paper.
\begin{lemma}\label{lem:matrixfrobenius}
    For any matrices $A$ and $B$, it holds that $\|AB\|_* \leq \|A\|_F \|B\|_F$.
\end{lemma}

\section{Exact Formulation of Induced Regularizer by Trace of Hessian}\label{sec:exactformulation}
In this section, we derive the exact formulation of trace of Hessian for $\ell_2$ loss over deep matrix factorization models with linear measurements as a minimization problem over $\mW$. We shall later approximate this formula by a different function in Section~\ref{sec:RIP}, which allows us to calculate the implicit bias in closed-form in the space of end-to-end matrices.

We first introduce the following simple lemma showing that the trace of the Hessian of the loss is equal to the sum of squares of norms of the gradients of the neural network output.
\begin{lemma}\label{lem:traceofhessian}
    For any twice-differentiable function $\{f_i(\mW)\}_{i=1}^n$,  real-valued labels $\{b_i\}_{i=1}^n$,  loss function $ \mathcal L(\mathbf W) = \frac{1}{n}\sum_{i=1}^n (f_i(\mW) - b_i)^2$, and any $\mW$ satisfying  $\gL(\mW)=0$, it holds that 
    \begin{align*}
         \tr(\nabla^2 \gL({\mathbf W})) = \frac{2}{n}\sum_{i=1}^n \|\nabla f_i(\mathbf W)\|^2.
    \end{align*}
\end{lemma}

Using Lemma~\ref{lem:traceofhessian}, we calculate the trace of Hessian for the particular loss defined in~\eqref{eq:lmp}. To do this, we consider $\mathbf W$ in Lemma~\ref{lem:traceofhessian} to be the concatenation of matrices $(W_1,\dots,W_L)$ and we set $f_i(\mW)$ to be the linear measurement $\inner{A_i}{E(\mW)}$, where $E(\mW) = W_L\cdots W_1$ (see~\eqref{eq:Edefinition}).
To calculate the trace of Hessian, according to Lemma~\ref{lem:traceofhessian}, we need to calculate the gradient of $\mathcal L(\mathbf W)$ in~\eqref{eq:lmp}. To this end, for a fixed $i$, we compute the gradient of $\inner{A_i}{E(\mW)}$ with respect to one of the weight matrices $W_j$.
\begin{align*}
    \nabla_{W_j} \inner{A_i}{E(\mW)} &= \nabla_{W_j} \tr(A_i^\top W_L\dots W_1) \\
    & = \nabla_{W_j} \tr((W_{j-1}\dots W_1 A_i^\top W_L\dots W_{j+1}) W_j)\\
    & = (W_{j-1}\dots W_1 A_i^\top W_L\dots W_{j+1})^\top.
\end{align*}
According to Lemma~\ref{lem:traceofhessian}, trace of Hessian is given by
\begin{align*}
    \tr(\nabla^2 L)(\textbf{W})  = \frac{1}{n}\sum_{i=1}^n \sum_{j=1}^L \|\inner{A_i}{E(\mW)}\|_F^2 =\frac{1}{n}\sum_{i=1}^n \sum_{j=1}^L \|W_{j-1}\dots W_1 A_i^\top W_L\dots W_{j+1}\|_F^2.
\end{align*}


As mentioned earlier, our approach is to characterize the minimizer of the trace of Hessian among all interpolating solutions by its induced regularizer in the end-to-end matrix space. The above calculation provides the following more tractable characterization of induced regularizer $F$ in \eqref{eq:Fdefinition}: 
\begin{align}\label{eq:F_jacobian}
F(M) = \min_{E(\mW) = M}\sum_{i=1}^n \sum_{j=1}^L \|W_{j-1}\dots W_1 A_i^\top W_L\dots W_{j+1}\|_F^2.	
\end{align}
In general, we cannot solve $F$ in closed form for general linear measurements $\{A_i\}_{i=1}^n$; however, interestingly, we show that it can be solved approximately under reasonable assumption on the measurements. In particular, we show that the induced regularizer, as defined in~\eqref{eq:F_jacobian}, will be approximately proportional to a power of the nuclear norm of $E(\mW)$ given that the measurements $\{A_i\}_{i=1}^n$ satisfy a natural norm-preserving property known as the Restricted Isometry Property (RIP)~\citep{candes2011tight,recht2010guaranteed}. 

Before diving into the proof of the general result for RIP,  we first illustrate the connection between nuclear norm and the induced regularizer for the depth-two case. In this case, fortunately, we can compute the closed form of the induced regularizer. This result was first proved by \citet{ding2022flat}. 
For self-completeness, we also provide a short proof.

%

\begin{theorem}[\citet{ding2022flat}]\label{thm:induced_regularizer_depth_2}
For any $M\in\mathbb{R}^{d_L\times d_0}$, it holds that 
    \begin{align}\label{eq:induced_regularizer_depth_2}
        F(M)\triangleq \min_{W_2 W_1 = M} \tr[\nabla^2\gL](\mW) = 2\norm{\left(\tfrac{1}{n}\sum\nolimits_i A_iA_i^\top \right)^{\nicefrac{1}{2} } M \left(\tfrac{1}{n}\sum\nolimits_i A_i^\top A_i \right)^{\nicefrac{1}{2}} }_*.
    \end{align}
\end{theorem}

\begin{proof}[Proof of \Cref{thm:induced_regularizer_depth_2}]
We first define $ {B_1} = (\sum_{i=1}^n A_i{A_i}^T)^{\frac{1}{2}}$ and $ {B_2} = (\sum_{i=1}^n {A_i}^TA_i)^{\frac{1}{2}}$.
Therefore we have that 
\begin{align*}
   \tr[\nabla^2\gL](\mW) = \sum_{i=1}^n \left( \|{A_i}^TW_2\|_F^2 + \|W_1{A_i}^T\|_F^2 \right)= \|B_1 W_2\|_F^2 + \|W_1B_2\|_F^2.
\end{align*}
Further applying \Cref{lem:matrixfrobenius}, we have that 
\begin{align*}
F(M)=& \min_{W_2 W_1 = M} \tr[\nabla^2\gL](\mW) =  \min_{W_2 W_1 = M} \sum_{i=1}^n \left( \|{A_i}^TW_2\|_F^2 + \|W_1{A_i}^T\|_F^2 \right) \\
\geq & \min_{W_2 W_1 = M} 2\|B_1W_2W_1 B_2\|_*^2
      =2\|B_1 M B_2\|_*^2.
\end{align*}
Next we show this lower bound of $F(M)$ can be attained. Let $U\Lambda V^T$ be the SVD  of $B_1 M B_2$. The equality condition happens for $W^*_2 = {B_1}^{\dagger}U\Lambda^{1/2},W^*_1 = \Lambda^{1/2}V^T {B_2}^{\dagger}$, where we have that $\sum_{i=1}^n \|{A_i}^TW^*_2\|_F^2 + \|W^*_1{A_i}^T\|_F^2 =2\|\Lambda\|_F^2 = 2\|B_1 M B_2\|_F^2$.
This completes the proof.
\end{proof}

The right-hand side in \eqref{eq:induced_regularizer_depth_2} will be very close to the nuclear norm of $M$ if the two extra multiplicative terms are close to the identity matrix. It turns out that $\{A_i\}_{i=1}^n$ satisfying the $(1,\delta)$-RIP exactly guarantees the two extra terms are $O(\delta)$-close to identity. However, the case for deep networks  where depth is larger than two is fundamentally different from the two-layer case, where one can obtain a closed form for $F$. To the best of our knowledge, it is open whether one obtain a closed form for the induced-regularizer for the trace of Hessian when $L > 2$. Nonetheless, in Section~\ref{sec:rip_preliminary}, we show that under RIP, we can still approximate it with the nuclear norm.


\section{Results for Measurements with Restricted Isometry Property (RIP)}\label{sec:RIP}
In this section, we present our main results for the generalization benefit of flatness regularization in deep linear networks. We structure the analysis as follows:
\begin{enumerate}[labelindent=10pt]
    \item In \Cref{sec:rip_preliminary}, we first recap some preliminaries on the RIP property.
    \item In \Cref{sec:induced_regularizer_rip}, we prove that the induced regularizer by trace of Hessian is approximately the power of nuclear norm for $(1,\delta)$-RIP measurements (\Cref{thm:RIPbias}).
    \item In \Cref{sec:recovering_groundtruth}, we prove that the minimum trace of Hessian interpolating solution with $(2,\delta)$-RIP measurements can recover the ground truth $M^*$ up to error $\delta \norm{M^*}_*^2$. For $\{A_i\}_{i=1}^n$ sampled from Gaussian distributions, we know $\delta = O(\sqrt{\frac{d_0+d_L}{n}})$. 
    \item  In \Cref{sec:fast_rate}, we prove a generalization bound with faster rate of $\frac{d_0+d_L}{n} \norm{M^*}_*^2$ using local Rademacher complexity based techniques from \citet{srebro2010smoothness}.
\end{enumerate}
Next, we discuss important distributions of measurements for which the RIP property holds.
\subsection{Preliminaries for RIP}\label{sec:rip_preliminary}
\begin{definition}[Restricted Isometry Property (RIP)]\label[definition]{defi:rip}
A family of matrices $\{A_i\}_{i=1}^n$ satisfies the $(r,\delta)$-RIP iff for any matrix $X$ with the same dimension and rank at most $r$:
\begin{align}
    (1-\delta)\|X\|_F^2\leq  \frac{1}{n}\sum\nolimits_{i=1}^n \langle A_i, X\rangle^2 \leq (1+\delta)\|X\|_F^2.\label{eq:ripdef}
\end{align}
	
\end{definition}


Next, we give two examples of distributions where $\Omega(r(d_0+d_L))$ samples guarantee $(r,O(1))$-RIP. 
The proofs follow from Theorem 2.3 in~\cite{candes2011tight}.
\begin{example}\label[example]{ex:gaussianexample}
Suppose for every $i\in \{1,\dots,n\}$, each entry in the matrix $A_i$ is an independent standard Gaussian random variable, \emph{i.e.}, $A_i \overset{i.i.d.}{\sim} \gG_{d_L \times d0}$. For every constant $\delta \in(0,1)$, if $n \geq \Omega(r(d_0+d_L))$, then with probability $1-e^{\Omega(n)}$,  $\{A_i\}_{i=1}^n$ satisfies $(r,\delta)$-RIP.

\end{example}
\begin{example}\label[example]{ex:bernoulli}
    If each entry of $A_i$ is from a symmetric Bernoulli random variable with variance $1$, i.e. for all $i,k,\ell$, entry $[A_i]_{k,\ell}$ is either equal to $1$ or $-1$ with equal probabilities, then for any $r$ and $\delta$, $(r,\delta)$-RIP holds with same probability as in  \Cref{ex:gaussianexample} if the same condition there is satisfied.
\end{example}

\subsection{Induced Regularizer of Trace of Hessian is Approximately Nuclear Norm}\label{sec:induced_regularizer_rip}

This section focuses primarily on the proof of \Cref{thm:ripgeneralization}. Our proof consists of two steps: (1) we show that the trace of Hessian of training loss at the minimizer $\mW$ is multiplicatively $O(\delta)$-close to the regularizer $R(\mW)$ defined below~(\Cref{lem:ripconcentration}) and (2) we show that the induced regularizer of $R$, $F'(M)$, is proportional to $\norm{M}_*^{2(L-1)/L}$ (\Cref{lem:Rclosedform}).
    \begin{align}
       \!\!\!\! R(\mW) \triangleq 
       \| W_L\dots W_{2}\|_F^2 d_0  + \sum_{j=2}^{L-1} \|W_{L}\dots W_{j+1}\|_F^2 \|W_{j-1}\dots W_{1}\|_F^2
        +\|W_{L-1}\dots W_1 \|_F^2 d_L.\label{eq:Rexplicitform}
    \end{align}
    
\begin{lemma}\label[lemma]{lem:ripconcentration} 
Suppose the linear measurement  $\{A_i\}_{i=1}^n$ satisfy $(1,\delta)$-RIP. Then, for any $\mW$ such that $\gL(\mW)=0$, it holds that 
\begin{align*}
    (1-\delta) R(\textbf{W}) \le  \textup{tr}(\nabla^2 L)(\textbf{W}) \leq (1+\delta) R(\textbf{W}).
\end{align*}
\end{lemma}
%

 Since $\textup{tr}(\nabla^2 \gL)(\mW)$ closely approximates $R(\mW)$, we can study $R$ instead of $\tr[\nabla^2 \gL]$ to understand the implicit bias up to a multiplicative factor $(1+\delta)$. In particular, we want to solve the induced regularizer of $R(\mW)$ on the space of end-to-end matrices, $F'(M)$: 
 \begin{align}
     F'(M)\triangleq \min_{\mW:W_L\cdots W_1 = M} R(\textbf W). \label{eq:induced_regularizer_population}
 \end{align}
 Surprisingly, we can solve this problem in closed form.
\begin{lemma}\label[lemma]{lem:Rclosedform}
    For any $M\in\mathbb{R}^{d_L\times d_0}$, it holds that 
    \begin{align}
    F'(M) \triangleq \min_{ \mW:  \ W_L\dots W_1=M} R(\textbf W) =
     L{(d_0 d_L)}^{1/L} \|M\|_*^{2(L-1)/L}.\label{eq:Fdefinition}
    \end{align}
\end{lemma}
\begin{proof}[Proof of \Cref{lem:Rclosedform}]
Applying the $L$-version of the AM-GM to Equation~\eqref{eq:Rexplicitform}:
\begin{align*}
    \left(R(\textbf W)/L\right)^L \geq & d_0 \|W_L\cdots W_{2}\|_F^2 \cdot \|W_1\|_F^2 \|W_L\cdots W_3\|_F^2\cdots   \|W_{L-1}\cdots  W_1\|_F^2 d_L.\numberthis\label{eq:AMGM} \\
    = & d_0 d_L  \prod_{j=1}^{L-1}\left( \norm{W_L\cdots W_{j+1}}_F^2\norm{W_j\cdots W_{1}}_F^2 \right)
\end{align*} 
Now using Lemma~\ref{lem:matrixfrobenius}, we have for every $1 \leq j \leq L-1$:
\begin{align}
    \|W_L\dots W_{j+1}\|_F^2\|W_{j}\dots W_1\|_F^2 \geq \|W_L\dots W_1\|_*^2  = \|M\|_*^2.\label{eq:frobholder}
\end{align}
Multiplying Equation~\eqref{eq:frobholder} for all $1 \leq j \leq L-1$ and combining with Equation~\eqref{eq:AMGM} implies 
\begin{align}
    \min_{ \{W| \ W_L\dots W_1=M\} } R(\textbf W)
    \geq L (d_0d_L)^{1/L}\|M\|_*^{2(L-1)/L}.\label{eq:minimumofR}
\end{align}
Now we show that equality can indeed be attained. To construct an example in which the equality happens, consider the singular value decomposition of $M$: $M = U \Lambda V^T$, where $\Lambda$ is a square matrix with dimension $\mathrm{rank}(M)$.

For $1\le i \le L$, we pick $Q_i\in \mathbb{R}^{d_i \times \mathrm{rank}(M)}$ to be any matrix with orthonormal columns. Note that $\mathrm{rank}(M)$ is not larger than $d_i$ for all $1\leq i\leq L$, hence such orthonormal matrices $Q_i$ exist. Then we define the following with $\alpha,\alpha'>0$ being constants to be determined:
\begin{align*}
    &W_L = \alpha' \alpha^{-(L-2)/2}U\Lambda^{1/2}{Q_{L-1}}^T \in \mathbb{R}^{d_L\times d_{L-1}},\\
    &W_{i} = \alpha Q_{i}{Q_{i-1}}^T \in \mathbb{R}^{d_i\times d_{i-1}}, \quad \forall 2\leq i \leq L-1,\\ 
    &W_1 = {\alpha'}^{-1} \alpha^{-(L-2)/2}Q_1\Lambda^{1/2} V^T \in \mathbb{R}^{d_1\times d_{0}}.
\end{align*}
Note that $\Lambda$ is a square matrix with dimension $\mathrm{rank}(M)$.
First of all, note that the defined matrices satisfy 
\begin{align*}
W_L W_{L-1}\dots W_1 &= \alpha^{L-2}\alpha^{-(L-2)}U\Lambda^{1/2}\Lambda^{1/2}V^T = M.
\end{align*}
To gain some intuition, we check that the equality case for all the inequalities that we applied above. We set the value of $\alpha$ in a way that these equality cases can hold simultaneously. 
Note that for the matrix holder inequality that we applied in Equation~\eqref{eq:frobholder}:
\begin{align*}
    \|W_L\dots W_{j+1}\|_F^2\|W_{j}\dots W_1\|_F^2 = \|W_L\dots W_1\|_*^2 = \|\Lambda^{1/2}\|_F^2,
\end{align*}
independent of the choice of $\alpha$. It remains to check the equality case for the AM-GM inequality that we applied in \Cref{eq:AMGM}. We have for all $2\leq j \leq L-1$:
\begin{align*}
    &\|W_L\dots W_{j+1}\|_F \|W_{j-1}\dots W_1\|_F \\
    &= \alpha^{j-2}\alpha^{-(L-2)/2}\alpha^{L-j-1}\alpha^{-(L-2)/2}\|U\Lambda^{1/2}\|_F\|\Lambda^{1/2}V^T\|_F = \alpha^{-1}\|\Lambda^{1/2}\|_F^2,\numberthis\label{eq:nonboundarycase}
\end{align*}
Hence, equality happens for all of them.
Moreover, for cases $j = 1$ and $j = L$, we have
\begin{align}
    &d_0\|W_L\dots W_2\| = \|\Lambda^{1/2}\|_F d_0 \alpha'\alpha^{L-2}\alpha^{-(L-2)/2} = \|\Lambda^{1/2}\|_F d_0\alpha'\alpha^{(L-2)/2}.\label{eq:boundarycasezero}\\
    &d_L\|W_{L-1}\dots W_{1}\| = \|\Lambda^{1/2}\|_F d_L{\alpha'}^{-1} \alpha^{L-2}\alpha^{-(L-2)/2} = \|\Lambda^{1/2}\|_F d_L{\alpha'}^{-1}\alpha^{(L-2)/2}.\label{eq:boundarycase}
\end{align}
Thus it suffices to  set $\alpha'=(\frac{d_L}{d_0})^{1/2}$ and  $\alpha = (\frac{\|\Lambda^{1/2}\|_F}{\sqrt{d_0d_L}})^{2/L} = (\frac{\|M\|_*}{d_0d_L})^{1/L}$ so that the left-hand sides of \eqref{eq:nonboundarycase}, \eqref{eq:boundarycasezero}, and \eqref{eq:boundarycase} are equal, which implies that the lower bound in Equation~\eqref{eq:minimumofR} is actually an equality. The proof is complete.
\end{proof}

Now we can prove \Cref{thm:RIPbias} as an implication of \Cref{lem:Rclosedform}.

\begin{proof}[Proof of \Cref{thm:RIPbias}]
The first claim is a corollary of \Cref{lem:ripconcentration}. 
We note that
\begin{align*}
    F(M) &= \min_{W_L\dots W_1 = M} \tr[\nabla^2 \mathcal L](M) \leq (1+\delta)\min_{W_L\dots W_1 = M} R(\mathbf W)= (1+\delta)F'(M) \\
    F(M) &= \min_{W_L\dots W_1 = M} \tr[\nabla^2 \mathcal L](M) \geq (1-\delta) \min_{W_L\dots W_1 = M} R(\mathbf W) = (1-\delta) F'(M).
\end{align*}
    For the second claim, pick $\bar{\mathbf W}$ that minimizes $R(\bar{\mathbf W})$ over all $\mathbf W$'s that satisfy the linear measurements, thus we have that
    \begin{align}\label{eq:R_Wmin}
        R(\bar{\mathbf W}) = L(d_0d_L)^{1/L}{\|E(\bar{\mW})\|_*}^{2(L-1)/L} = L(d_0d_L)^{1/L}{\min_{ \gL'(M)=0}\|M\|_*}^{2(L-1)/L}.
    \end{align}
    
    Now from the definition of $E(\mW^*)$, 
    \begin{align}
        \textup{tr}(\nabla^2 L)(\mathbf W^*) \leq \textup{tr}(\nabla^2 L)(\bar{\mathbf W})\leq (1+\delta)R(\bar{\mathbf W}),\label{eq:aval}
    \end{align}
    where the last inequality follows from the definition of $W$. On the other hand
    \begin{align}
        \textup{tr}(\nabla^2 L)({\mathbf W}^*) \geq (1-\delta)R(\bar{\mathbf W}) \geq (1-\delta)L(d_0 d_L)^{1/L}{\|E({\mathbf{W}}^*)\|_*}^{2(L-1)/L}.\label{eq:dovom}
    \end{align}
    Combining \eqref{eq:R_Wmin}, \eqref{eq:aval} and \eqref{eq:dovom},
    \begin{align*}
        \|E({\mathbf{W}}^*)\|_* \leq (\frac{1+\delta}{1-\delta})^{\frac{L}{2(L-1)}} \min_{ \gL'(M)=0}\|M\|_*.
    \end{align*}
    The proof is  completed by noting that $\frac{L}{2(L-1)}\leq 1$ for all $L\ge 2$. 
\end{proof}

Thus combining \Cref{ex:gaussianexample} and \Cref{thm:RIPbias} with $\delta =1/2$, we have the following corollary. 
\begin{corollary}\label[corollary]{cor:high_prob_min_trace_of_hessian}
	Let $\{A_i\}_{i=1}^n$ be sampled independently from Gaussian distribution $\gG_{d_L \times d_0}$ where $n\ge \Omega((d_0+d_L))$, with probability at least $1-\exp(\Omega(n))$, we have 
	    \begin{align*}
        \| E(\mW^*)\|_* \leq 3 \min_{ \gL'(M)=0}\|M\|_* \le 3 \norm{ E(\mW^*)}_*.
    \end{align*}
\end{corollary}

\subsection{Recovering the Ground truth}\label{sec:recovering_groundtruth}
In this section, we prove Theorem~\ref{thm:ripgeneralization}. The idea is to show that under RIP, the empirical loss $\gL(\mW)$ is a good approximation for the Frobenius distance of $E(\mW)$ to the ground truth $M^*$. To this end, we first introduce a very useful Lemma~\ref{lem:RIPrelax} below, whose proof is deferred to \Cref{sec:omitted_proofs}.

\begin{lemma}\label[lemma]{lem:RIPrelax}
    Suppose the measurements $\{A_i\}_{i=1}^n$ satisfy the $(2,\delta)$-RIP condition. Then for any matrix $M\in\mathbb{R}^{d_L\times d_0}$, we have that 
    \begin{align*}
        \Big| \frac{1}{n} \sum\nolimits_{i=1}^n \inner{A_i}{M}^2- \|M\|_F^2 \Big| \leq 2 \delta\|M\|^2_*.
    \end{align*}
\end{lemma}
We note that if $\{A_i\}_{i=1}^n$ are i.i.d. random matrices with each coordinate being independent, zero mean, and unit variance (like standard Gaussian distribution), then $\|W - M^*\|_F^2$ is the population squared loss corresponding to $W$. Thus, \Cref{thm:ripgeneralization} implies a generalization bound for this case. Now we are ready to prove \Cref{thm:ripgeneralization}.
\begin{proof}[Proof of \Cref{thm:ripgeneralization}]
       Note that from \Cref{thm:RIPbias},
    \begin{align*}
        \| E(\mW)\|_* \leq \frac{1 + \delta}{1-\delta} \min_{ \gL'(M)=0}\|M\|_* \leq \frac{1+\delta}{1-\delta}\|M^*\|_*,
    \end{align*}
    which implies the following by triangle inequality,
    \begin{align}
        \| E(\mW) - M^*\|_* &\leq \|\tilde E(\mW)\|_* + \| M^*\|_*
        \leq \frac{2}{1-\delta}\| M^*\|_*.\label{eq:secondeq}
    \end{align}
    Combining~\eqref{eq:secondeq} with Lemma~\ref{lem:RIPrelax} (with $M =E(\mW) - M^*$):
    \begin{align*}
         \Big| \frac{1}{n} \sum\nolimits_{i=1}^n \inner{A_i}{E(\mW^*) - M^*}^2 - \|E(\mW^*) - M^*\|_F^2 \Big| \leq  \frac{8\delta}{(1-\delta)^2}\| M^*\|_*^2.
    \end{align*}
    Since $W^*$ satisfies the linear constraints $\tr(A_i E(\mathbf W^*)) = b_i$, $\frac{1}{n} \sum_{i=1}^n \inner{A_i}{E(\mW^*) - M^*}^2 =  \frac{1}{n} \sum_{i=1}^n \big(\inner{A_i}{E(\mW^*)} - b_i\big)^2 = 0$, which completes the proof.
\end{proof}


\subsection{Generalization Bound}\label{sec:fast_rate}
In this section, we prove the generalization bound in Theorem~\ref{thm:generalizationfastrate}, which yields a faster rate of $O(\frac{d_0+d_L}{n}\norm{M^*}_*^2)$ compared to $O(\sqrt{\frac{d_0+d_L}{n}}\norm{M^*}_*^2)$ in Theorem~\ref{thm:ripgeneralization}. The intuition for this is as follows: By \Cref{cor:high_prob_min_trace_of_hessian}, we know that with very high probability, the learned solution has a bounded nuclear norm for its end-to-end matrix, no larger than $3\norm{M^*}_2$, where $M^*$ is the ground truth.
The key mathematical tool is \Cref{thm:fast_rate_rademacher_complexity}, which provides an upper bound on the population error of the learned interpolation solution that is proportional to the square of the Rademacher complexity of the function class $\gH_{3\|M^*\| _*} = \{h_M\mid \norm{M}_*\le 3\norm{M^*}_*\}$.

\begin{theorem}[Theorem 1, \citet{srebro2010smoothness}]\label[theorem]{thm:fast_rate_rademacher_complexity}
	Let $\gH$ be a class of real-valued functions and $\ell:\mathbb{R}\times\mathbb{R}\to \mathbb{R}$ be a differentiable non-negative loss function satisfying that (1) for any fixed $y\in\mathbb{R}$, the partial derivative $\ell(\cdot,y)$ with respect to its first coordinate is $H$-Lipschitz and (2) $|\sup_{x,y}\ell(x,y)|\le B$, where $H,B$ are some positive constants. Then for any $p>0$, we have that with probability at least $1-p$ over a random sample of size $n$, for any $h\in\gH$ with zero training loss,
	\begin{align}\label{eq:fast_rate}
	\Lbar(h)\le O\left( H\log^3 n \gR_n^2(\gH) + \frac{B\log (1/p)}{n}\right).
	\end{align}

\end{theorem}

One technical difficulty is that \Cref{thm:fast_rate_rademacher_complexity} only works for bounded loss functions, but the $\ell_2$ loss on Gaussian data is unbounded. To circumvent this issue, we construct a smoothly truncated variant of $\ell_2$ loss~\eqref{eq:truncated_smooth_l2_loss} and apply \Cref{thm:fast_rate_rademacher_complexity} on that. Finally, we show that with a carefully chosen threshold, this truncation happens very rarely and, thus, does not change the population loss significantly.  The proof can be found in \Cref{sec:omitted_proofs}.

\section{Result for the Single Measurement Case}\label{sec:singlemeasurement}
Quite surprisingly, even though in the general case we cannot compute the closed-form of the induced regularizer in~\eqref{eq:Fdefinition}, we can find its minimum as a quasinorm function of the $E(\mW)$ which only depends on the singular values of $E(\mW)$. This yields the following result for multiple layers $L$ (possibly $L > 2$) with a single measurement. 

%

\begin{theorem}\label{thm:induced_regularizer_single_measurement}
    Suppose there is only a single measurement matrix $A$, \emph{i.e.}, $n=1$. For any $M\in\mathbb{R}^{d_L\times d_0}$, the following holds:
    \begin{align}
        F(M) = \min_{W_L\dots W_1 = M} \tr[\nabla^2 \gL](\mW)= L\snorm{\Big(A^T M\Big)^{L-1}A^T}{2/L}^{2/L}.\label{eq:singlemeasurementcase}
    \end{align}
\end{theorem}
To better illustrate the behavior of this induced regularizer, consider the case where the measurement matrix $A$ is identity and $M$ is symmetric with eigenvalues $\{\sigma_i\}_{i=1}^d$. Then, it is easy to see that $F(M)$ in~\eqref{eq:singlemeasurementcase} is equal to $F(M) = \sum_i \sigma_i^{2(L-1)/L}$. Interestingly, we see that the value of $F(M)$ converges to the Frobenius norm of $M$ and not the nuclear norm as $L$ becomes large, which behaves quite differently (e.g. in the context of sparse recovery). This means that beyond RIP, the induced regularizer can behave very differently, and perhaps the success of training deep networks with SGD is closely tied to the properties of the dataset. 

\section{Experiments}
In this section, we examine our theoretical results with controlled experiments via synthetic data. The experiments are based on mini-batch SGD and label noise SGD~\citep{blanc2019implicit}. Both use  the standard update rule $\mW_{t+1} = \mW_t - \eta \nabla \gL_t(\mW_t)$, but with different objectives:
\begin{itemize}
	\item Mini-batch loss: $ \gL_t^{\textup{mini-batch}}(\mW) = \frac{1}{B}\sum_{i\in \mathcal B_t} (f_i(\mW) - b_i)^2$;
	\item Label-noise loss: $\gL_t^{\textup{label-noise}(\mW)} = \frac{1}{B}\sum_{i\in \mathcal B_t} (f_i(\mW) - b_i + \xi_{t,i})^2$,
\end{itemize}
where $\mathcal B_t$ is the batch of size $B$ independently sampled with replacement at step $t$ and  $\xi_{t}\in\mathbb{R}^d$ are i.i.d. multivariate zero-mean Gaussian random variables with unit variance.

It is known that with a small learning rate, label noise SGD implicitly minimizes the trace of Hessian of the loss, after reaching zero loss~\citep{damian2021label,li2021happens}. In particular,~\citet{li2021happens} show that after reaching zero loss, in the limit of step size going to zero, label noise SGD converges to a gradient flow according to the negative gradient of the trace of Hessian of the loss. As a result, we expect label noise SGD to be biased to regions with smaller trace of Hessian. We also compare the label noise SGD with vanilla SGD without label noise as a baseline, which can potentially find a solution with large sharpness when the learning rate is small. Note this is not contradictory with the common belief that mini-batch SGD prefers flat minimizers and thus benefits generalization~\citep{keskar2016large,jastrzkebski2017three}. For example, assuming the convergence of mini-batch SGD, \citep{wu2018sgd} shows that the solution found by SGD must have a small sharpness, bounded by a certain function of the learning rate. However, there is no guarantee when the learning rate is small and the upper bound of sharpness becomes vacuous.

\begin{figure*}[t]
\begin{center}   
\includegraphics[width=\linewidth]{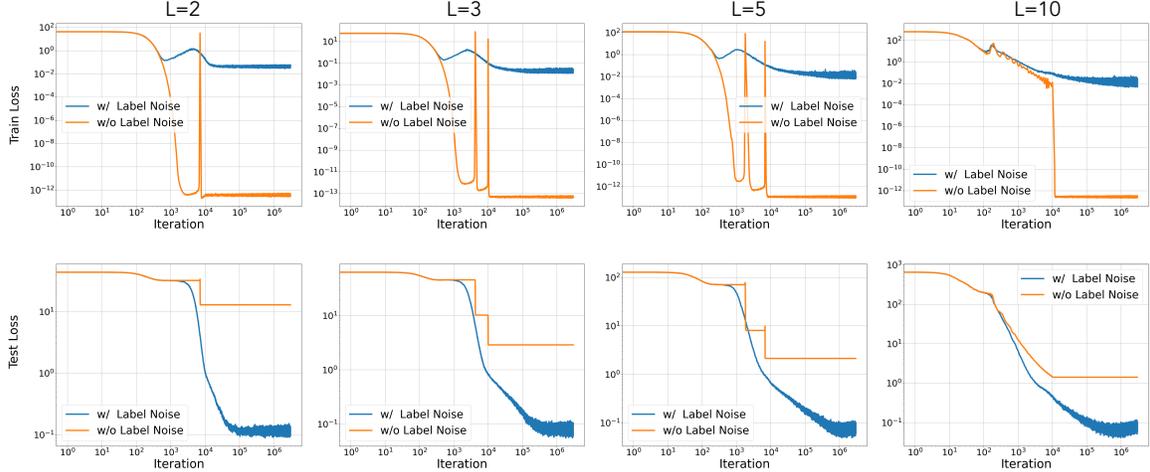}
\end{center}
\vspace{-2mm}
\caption{\textbf{Train and test loss.} Label noise SGD leads to better generalization results due to the sharpness-minimization implicit biases (as shown in Figure \ref{fig_2}), while mini-batch SGD without label noise finds solutions with much larger test loss.} \label{fig_1}
\vspace{-0mm}
\end{figure*}

\begin{figure*}[t]
\begin{center}   
\includegraphics[width=\linewidth]{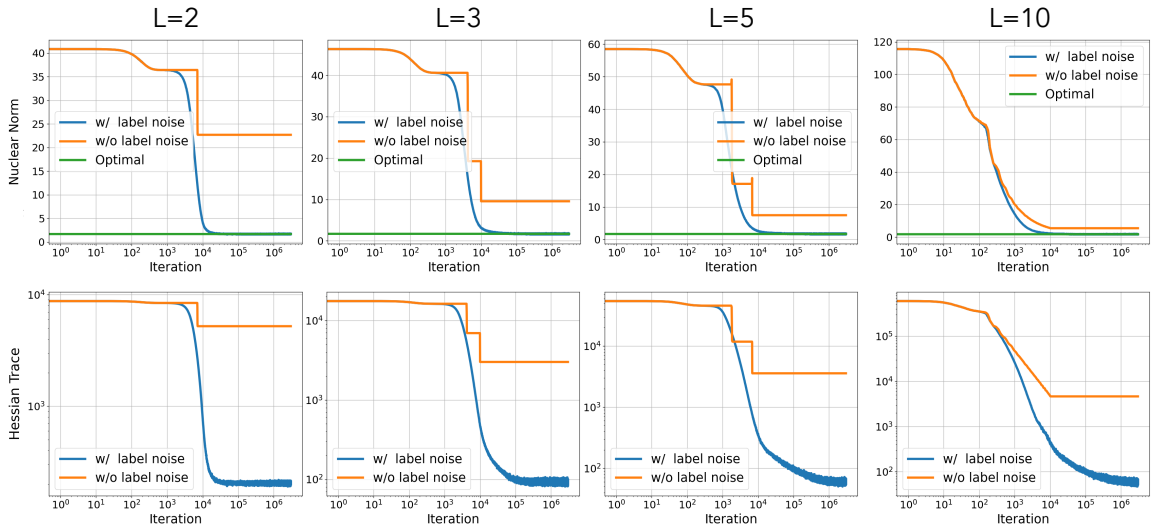}
\end{center}
\vspace{-2mm}
\caption{\textbf{Trace of Hessian and Nuclear Norm.} Label noise SGD  recovers the min nuclear norm solution via its sharpness-minimization implicit regularization and thus leads to better generalization (see \Cref{fig_1}).} \label{fig_2}
\vspace{-4mm}
\end{figure*}

In our synthetic experiments, we sample $n=600$ input matrices $\{A_i\}_{i=1}^n$, where $A_i \in \sR^{d \times d}$ with $d=60$. Each entry $A_i^{(j,k)}$ is i.i.d. sampled from normal distribution $\gN(0,1)$. The ground truth matrix $M^\ast$ is constructed by $M^\ast = M_1M_2 / d$, where $M_1 \in \sR^{d \times r}$ and $M_2 \in \sR^{r \times d}$ and $r$ is the rank of $M^\ast$. The entries in $M_1$ and $M_2$ are again i.i.d. sampled from $\gN(0,1)$ and the rank $r$ is set to $3$. The corresponding label is therefore computed via $b_i = \langle A_i, M^\ast \rangle$. The parameters $(W_1, . . . , W_L)$ are sampled from a zero-mean normal distribution for depth $L=2,3,5,$ and $10$. For label noise SGD, we optimize the parameter via SGD with label noise drawn from $\gN(0, 1)$ and batch size $50$. The learning rate is set to $0.01$.

We examine our theory by plotting the training and testing loss along with the nuclear norm and the trace of Hessian of the label noise SGD solutions in  \Cref{fig_1,fig_2}. As the figure illustrates, the trace of the Hessian exhibits a gradual decrement, eventually reaching a state of convergence over the course of the training process. This phenomenon co-occurs with the decreasing of the nuclear norm of the end-to-end matrix. In particular, we further plot the nuclear norm of the min nuclear norm solution obtained via solving convex optimization in \Cref{fig_2} and demonstrate that label noise SGD converges to the minimal nuclear norm solution, as predicted by our theorem~\Cref{thm:RIPbias}. As a consequence of this sharpness-minimization implicit bias, the test loss decreases drastically.

Interestingly, there are a few large spikes in the training loss curve of  mini-batch SGD without label noise even after the training loss becomes as small as $10^{-12}$ and its generalization improves immediately after recovering from the spike. Meanwhile, the trace of hessian and the nuclear decrease during this process. We do not have a complete explanation for such spikes. One possible  explanation from the literature~\citep{ma2018implicit} is that the loss landscape around the minimizers is too sharp and thus mini-batch SGD is not linear stable around the minimizer, so it escapes eventually. However, this explanation does not explain why minibatch SGD can find a flatter minimizer each time after escaping and re-converging.

\section{Conclusion and Future Directions}\label{sec:conclusion}



In this paper, we study the inductive bias of the minimum trace of the Hessian solutions for learning deep linear networks from linear measurements. We show that trace of Hessian regularization of loss on the end-to-end matrix of deep linear networks roughly corresponds to nuclear norm regularization under restricted isometry property (RIP) and yields a way to recover the ground truth matrix. Furthermore, leveraging this connection with the nuclear norm regularization, we show a generalization bound which yields a faster rate than Frobenius (or $\ell_2$ norm) regularizer for Gaussian distributions. Finally, going beyond RIP conditions, we obtain closed-form solutions for the case of a single measurement. Several avenues for future work remain open, e.g., more general characterization of trace of Hessian regularization beyond RIP settings and understanding it for neural networks with non-linear activations.


\section*{Acknowledgement}
TM and ZL would like to thank the support from NSF IIS 2045685.

\bibliographystyle{plainnat}
\bibliography{references,all}

\newpage
\appendix

\section{Proof of Lemma~\ref{lem:ripconcentration}}
\begin{proof}[Proof]
For a fixed $j \in \{2,\dots, L-1\}$ and vectors $x \in \mathbb R^{d_0}$ and $y \in \mathbb R^{d_L}$ we apply the RIP property in~\Cref{defi:rip} for the rank one matrix $X = xy^T$. As a result we get
\begin{align*}
    (1-\delta)\|xy^T\|_F^2 \leq \frac{1}{n}\sum_{i=1}^n \langle A_i, xy^T\rangle^2 \leq (1+\delta)\|xy^T\|_F^2,
\end{align*}
or equivalently
\begin{align}
    (1-\delta)\|x\|^2\|y\|^2 \leq \frac{1}{n}\sum_{i=1}^n (x^TA_iy)^2 \leq (1+\delta)\|x\|^2\|y\|^2.\label{eq:RIPvectors}
\end{align}

Now for arbitrary indices $1\leq \ell \leq d_{j-1}$ and $1 \leq k \leq d_j$, we pick $x,y$ in Equation~\eqref{eq:RIPvectors} equal to the $\ell$th row of the matrix $W_{j-1}\dots W_1$ and the $k$th column of the matrix $W_L\dots W_{j+1}$:
\begin{align*}
    &(1-\delta)\|(W_{j-1}\dots W_1)_{\ell:}\|^2\|(W_L\dots W_{j+1})_{:k}\|^2 \\
    &\leq \frac{1}{n}\sum_{i=1}^n ((W_{j-1}\dots W_1)_{\ell:}A_i(W_L\dots W_{j+1})_{:k})^2 \\
    &\leq (1+\delta)\|W_{j-1}\dots W_1)_{\ell:}\|^2\|(W_L\dots W_{j+1})_{:k}\|^2.\numberthis\label{eq:RIPentrywise}
\end{align*}
Summing this over all $\ell,k$, we obtain that the sum of Frobenius norm of matrices $W_{j-1}\dots W_1 A_i W_L \dots W_{j+1}$ concentrate around $\|W_{j-1}\dots W_1\|_F^2 \|W_L\dots W_{j+1}\|_F^2$. 
\begin{align*}
    &(1-\delta)\|W_{j-1}\dots W_1\|_F^2\|W_L\dots W_{j+1}\|_F^2 \\
    &\leq \frac{1}{n}\sum_{i=1}^n \|W_{j-1}\dots W_1A_iW_L\dots W_{j+1}\|_F^2 \\
    &\leq (1+\delta)\|W_{j-1}\dots W_1\|^2\|W_L\dots W_{j+1}\|^2.\numberthis\label{eq:frobnorm}
\end{align*}
For $j=1$, we apply Equation~\eqref{eq:RIPvectors} with $x=(W_{j-1}\dots W_1)_{\ell:}$ and $y=e_k$, where $e_k$ is the $k$th standard vector:
\begin{align*}
    (1-\delta)\|(W_{L-1}\dots W_1)_{\ell:}\|^2 \leq \frac{1}{n}\sum_{i=1}^n ((W_{L-1}\dots W_1)_{\ell:} A_i e_k)^2 \leq (1+\delta)\|(W_{L-1}\dots W_1)_{\ell:}\|^2.
\end{align*}
Summing this for all $k,\ell$
\begin{align}
    (1-\delta)d_0\|W_{L-1}\dots W_1\|_F^2 \leq \frac{1}{n}\sum_{i=1}^n \|W_{L-1}\dots W_1 A_i\|_F^2 \leq (1+\delta)d_0\|W_{L-1}\dots W_1\|_F^2.\label{eq:jonecase}
\end{align}
Similarly for $j = L$, 
\begin{align}
    (1-\delta)d_L\|W_{L}\dots W_2\|_F^2 \leq \frac{1}{n}\sum_{i=1}^n \|A_i W_{L}\dots W_2\|_F^2 \leq (1+\delta)d_L\|W_{L}\dots W_2\|_F^2.\label{eq:jLcase}
\end{align}
Combining Equations~\eqref{eq:frobnorm},~\eqref{eq:jonecase}, and~\eqref{eq:jLcase} 
\begin{align*}
    (1-\delta)R(W) \leq \textup{tr}(\nabla^2 L)(W) \leq (1+\delta)R(W).
\end{align*}
\end{proof}

\section{Proof of \Cref{thm:induced_regularizer_single_measurement}}
\begin{proof}[Proof of \Cref{thm:induced_regularizer_single_measurement}]
Recall that we hope to characterize the solution with a minimal trace of hessian given that the end-to-end matrix $E(\mW) = W_L\cdots W_1$ is equal to some fixed matrix $M$, namely,
\begin{align*}
    &\min_{E(\mW)=M} \sum_{i=1}^L \|W_{i-1}\dots W_1 A^T W_L\dots W_{i+1}\|_F^2.
\end{align*}
Let $\mW$ be any minimizer of the above objective. For arbitrary matrix $C \in {\mathbb R}^{d_i \times d_i}$, define
\begin{align*}
    U(t) = \exp(tC) \triangleq \sum_{i=0}^\infty \frac{(tC)^i}{i!},
\end{align*}
For any $i$, we multiply $W_i$ from left by $U(t)$ and multiply $W_{i+1}$ by $U(t)^{-1}$ from right, 
\begin{align*}
    &W_i(t) \leftarrow U(t) W_i,\\
    &W_{i+1}(t) \leftarrow W_{i+1}U(t)^{-1}.
\end{align*}
For convenience, below we drop the dependence of $W_i(t),W_{i+1}(t)$ over $t$, that is, only $W_i$ and $W_{i+1}$ are implicitly functions of $t$, while the rest $W_j$ are independent of $t$. Then, note that for any $j\le i-1$ we have
\begin{align*}
    W_{j-1}\dots W_1 A^T W_L\dots W_{i+1}U(t)^{-1}U(t)W_{i} \dots W_{j+1}= W_{j-1}\dots W_1 A^T W_L \dots W_{j+1},
\end{align*}
and for $j\ge i+2$:
\begin{align*}
    W_{j-1}\dots W_{i+1}{U(t)}^{-1}U(t)W_{i}\dots W_1A^T W_L\dots W_{j+1} = W_{j-1}\dots W_{i+1}W_{i}\dots W_1A^T W_L\dots W_{j+1}.
\end{align*}
So the only terms that actually change as a function of $t$ correspond to $j = i$,
\begin{align}
    \|W_{i-1}\dots W_1 A^T W_L\dots W_{i+1}\|_F^2 
    =\tr(W_{i-1}\dots W_1A^TW_L \dots W_{i+1} {W_{i+1}}^T\dots W_L^T A {W_1}^T\dots {W_{i-1}}^T),\label{eq:firstt}
\end{align}
and to $j=i+1$,
\begin{align}
    \|W_{i}\dots W_1 A^T W_L\dots W_{i+2}\|_F^2 = \tr(W_{i}\dots W_1 A^T W_L\dots W_{i+2}{W_{i+2}}^T\dots {W_L}^T A {W_1}^T \dots {W_{i}}^T).\label{eq:secondd}
\end{align}
Now taking derivative of $U(t)$ with respect to $t$, 
\begin{align*}
    U'(0) = C.
\end{align*}
Now for every $j \in \{1,\dots,L\}$ we define
\begin{align*}
    \widetilde W_j = W_{j-1}\dots W_1 A^T W_L\dots W_{j+1},
\end{align*}
where we use $W_{i-1}\dots W_1$ and $W_L\dots W_{i+1}$ to denote identity for $i = 1$ and $i=L$ respectively.

Then, if we take derivative from the terms~\eqref{eq:firstt} and~\eqref{eq:secondd} with respect to $t$:
\begin{align}
    &\frac{d}{dt} \|W_{i-1}\dots W_1 A^T W_L\dots W_{i+1}\|_F^2\Big|_{t=0}\notag\\
    &=
    -\tr((C + C^T){W_{i+1}}^T\dots {W_L}^TA{W_1}^T\dots {W_{i-1}}^T W_{i-1}\dots W_1A^T W_L\dots W_{i+1}),\\
    &=\tr((C+C^T){\widetilde W_i}^T {\widetilde W_i}).\notag
\end{align}
and
\begin{align}
    &\frac{d}{dt} \|W_{i}\dots W_1 A^T W_L\dots W_{i+2}\|_F^2\Big|_{t=0}\notag\\
    &=-\tr((C+C^T){\widetilde W_{i+1}}{\widetilde W_{i+1}}^T)
\end{align}
Now from the optimality of $\mW$, the following equality holds for every matrix $C\in\mathbb{R}^{d\times d}$:
\begin{align}
    \frac{\diff}{\diff t} \textup{tr}[\nabla^2 \gL(\mW(t))]\Big|_{t=0} = \tr((C+C^T)({\widetilde W_i}^T {\widetilde W_i} - {\widetilde W_{i+1}}{\widetilde W_{i+1}}^T)) = 0.\label{eq:corekkt}
\end{align}
Now since $C$ is arbitrary and the matrices 
${\widetilde W_i}^T \widetilde W_i$ and $\widetilde W_{i+1}{\widetilde W_{i+1}}^T$ are symmetric, we must have
\begin{align}\label{eq:i_i+1_simultaneous_diagonalizable}
    {\widetilde W_i}^T {\widetilde W_i} = {\widetilde W_{i+1}}{\widetilde W_{i+1}}^T.
\end{align}
\Cref{eq:i_i+1_simultaneous_diagonalizable} implies that all $\widetilde W_{i}$ for $1\le i \le L$ have the same set of singular values. Moreover, there exists matrices $\{U_i\}_{i=0}^{L}$ where the columns of each matrix are orthogonal, such that for each $1\le i\le L$,
\begin{align}
    \widetilde{W_i} = W_{i-1}\dots W_1A^TW_L \dots W_{i+1} = U_{i-1} \Lambda {U_{i}}^T.\label{eq:symmetric2}
\end{align}
Multiplying Equation~\eqref{eq:i_i+1_simultaneous_diagonalizable} for all $1 \leq i \leq L$ (in the case $i=1$ we take $W_1\dots W_{i-1}$ as identity), we get
\begin{align}
    \Big(A^T E(\mathbf W)\Big)^{L-1}A^T = \Big(A^T W_L\dots W_1\Big)^{L-1}A^T = U_{0}\Lambda^L {U_L}^T,\label{eq:keyeqnonsymmetric}
\end{align}
or in case where $A$ is positive semi-definite,
\begin{align}
    A^{1/2}\Big(A^{1/2}E(\mathbf W) A^{1/2}\Big)^{L-1}A^{1/2} = U_{0}\Lambda^L {U_L}^T.\label{eq:keyeqsymmetric}
\end{align}

But having access to Equations~\eqref{eq:symmetric2}, we can write $\textup{tr}[\nabla^2 \gL(\mW)]$ at the minimizer point $\mW = (W_1, \dots, W_L)$ as
\begin{align*}
    \sum_{i=1}^L \|W_{i-1}\dots W_1 A W_L\dots W_{i+1}\|_F^2 = L\|\Lambda\|_F^2 = L\|\Lambda^L\|_{S_{2/L}}^{2/L}.
\end{align*}
which based on Equation~\eqref{eq:keyeqnonsymmetric} is equal to
\begin{align*}
    L\snorm{\Big(A^T E(\mathbf W)\Big)^{L-1}A^T}{2/L}^{2/L},
\end{align*}
or in the symmetric case is equal to
\begin{align*}
    L\snorm{A^{1/2}\Big(A^{1/2}E(\mW) A^{1/2}\Big)^{L-1}A^{1/2}}{2/L}^{2/L}.
\end{align*}
This is the induced regularizer of the trace of Hessian over all interpolating solutions for linear network with depth $L$ in the space of end-to-end matrices.

\end{proof}
\section{Other Omitted Proofs}\label{sec:omitted_proofs}
\subsection{Proof of \Cref{thm:generalizationfastrate}}
\begin{proof}[Proof of \Cref{thm:generalizationfastrate}]
    By \Cref{cor:high_prob_min_trace_of_hessian}, we know that  with probability at least $1-\exp(\Omega(n))$,
    \begin{align*}
        \|E(\mW^*)\|_* \leq 3\|M^*\|_*. 
    \end{align*}
	Note by assumption, $n=\Omega(d_0+d_L)$. Thus it suffices to show that with probability at least $1-\exp(\Omega(d_0+d_L))$, for all interpolating solutions in $\gH_{3\|M^*\| _*} = \{h_M\mid \norm{M}_*\le 3\norm{M^*}_*\}$, \Cref{eq:fast_rate} holds. 

    Recall $\Lbar \big(E(\mW)\big)$ is the population square loss at the end-to-end matrix $E(\mW)\in \mathbb R^{d_0\times d_L}$. Namely, 
    \begin{align*}
        \mathcal \Lbar(E(\mW)) \triangleq \mathbb E_{A }(\inner{ A}{ E(\mW^*) } - \inner{ A}{ M^* }) ^2 = \mathbb E_{A} \inner{A}{E(\mW)-M^*}^2 = \norm{E(\mW^*) - M^*}_F^2.
    \end{align*}

    First, we bound the population Rademacher complexity of function classs $\mathcal \gH_{3\norm{M^*}_*}$. Its empirical Rademacher complexity on $\{A_i\}_{i=1}^n$is 
    \begin{align*}
        \mathcal R_n(\gH_{3\norm{M^*}_*}) &=\frac{1}{n}\mathbb E_{\epsilon \sim \{\pm 1\}^n} \sup_{h\in \gH_{3\norm{M^*}_*}}\sum_{i=1}^n  \epsilon_i h(A_i)\\ 
        &= \frac{1}{n}\mathbb
         E_{\epsilon \sim \{\pm 1\}^n} \sup_{M: \|M\|_*
        \leq 3\|M^*\|_*}\sum_{i=1}^n \langle \epsilon_i A_i, M\rangle = 3/n \cdot\|M^*\|_*\|\sum_{i=1}^n \epsilon_i A_i\|_{2}.
    \end{align*}
     Note that the matrix $A_{sum} = \sum_{i=1}^n \epsilon_i A_i$ itself is an iid Gaussian matrix where each entry is sampled from $\mathcal N(0,n)$. Hence, from Proposition 2.4 in~\cite{rudelson2010non}, we have the following tail bound on the spectral norm of $A_{sum}$
    \begin{align}
        \mathbb P(\|A_{sum}\|_2 \geq c_1\sqrt n (\sqrt{d_0} + \sqrt{d_L}) + \sqrt n t) \leq 2 e^{-c_2 t^2},\label{eq:tailbound}
    \end{align}
    This implies $\mathbb E \|A_{sum}\|_2 = O\big(\sqrt n(\sqrt{d_0} + \sqrt{d_L})\big)$,
    which in turn bounds the Rademacher complexity
    \begin{align}
        \overline\gR_n(\gH_{3\norm{M^*}_*}) = \Exp \gR_n(\gH_{3\norm{M^*}_*}) =O\left( \frac{\sqrt{d_0} + \sqrt{d_L}}{\sqrt n}\norm{M^*}_*\right).\label{eq:rademacher}
    \end{align}
    
    Note that the Gaussian distribution $\gG_{d_L\times d_0}$ is unbounded, which makes the value of the squared loss unbounded, while it is convenient to bound the generalization gap when the value of the loss is bounded. To cope with this fact, for a given threshold $c$, we define a truncated version of the loss denoted by $l_c(x,y) = \ell_c(x-y)$, plotted in Figure~\ref{fig:surrogate}, which is a smooth approximation of the squared loss. 
    \begin{align}\label{eq:truncated_smooth_l2_loss}
    l_c(x,y) = \ell_c(x-y) =
    \begin{cases}
          (x-y)^2,  & \textup{if }x-y \in [-c, c],\\
           -(x-y)^2 + 4c|x-y| - 2c^2, & \textup{if } x-y \in [-2c,-c]\cup [c,2c],\\
          2c^2, & \textup{if } x-y\in (-\infty, -2c]\cup [2c, \infty).
    \end{cases}
    \end{align}
     It is easy to verify $\partial_x\ell_c$ is $2$-lipschitz in $x$.
    Also it is clear that $l_c(x,y)\le \max(2c^2,l(x,y))$ for all $x,y$ and $l_c(x,y)< l(x,y)$  only when $|x-y|>c$. Next, we define the $c$-cap population loss $\Lbar_c$ with respect to $\ell_c$:
    \begin{align*}
        \Lbar_c(M) = \mathbb E_{A\sim \gG_{d_L\times d_0}} \ell_c(\inner{A}{M}, \inner{A}{M^*}).
    \end{align*}
Thus we have 
    \begin{align*}
        \Lbar(M) - \Lbar_c(M) &\leq
        \mathbb E \ind_{|\inner{A}{M-M^*}|\ge c}\inner{A}{M-M^*}^2.
    \end{align*}
    But note that the variable $\inner{A}{M-M^*}$ is a Gaussian variable with variance $\|M - M^*\|^2_F \leq \|M - M^*\|_*^2$. Hence, from Lemma~\ref{lem:forthmomentbound}, picking $c = \Theta(\log(n)\|M^*\|_*)$, for all $M\in\gH_{3\norm{M^*}_*}$, 
    \begin{align*}
        0\le \Lbar(M) - \Lbar_c(M) \le   O(\frac{\|M^*\|_*^2 \log n }{n}).
    \end{align*}    
    
    Now using the Rademacher complexity bound in~\eqref{eq:rademacher} and applying  \Cref{thm:fast_rate_rademacher_complexity}, we have for all interpolating solutions $M\in\gH_{3\norm{M^*}_*}$ with probability at least $1-\exp(\Omega(d_0+d_L))$:
    \begin{align*}
        \Lbar_c(M)  &\leq O\left( H \log^3(n) \mathcal R_n^2 +\frac{c^2(d_0+d_L)}{n}\right) \\
        &\leq O\left(\norm{M^*}_*^2 \frac{(d_0+d_L)\log^3 n}{n}\right)
        \numberthis\label{eq:upperbound1}
    \end{align*}
    
    where $H$ is the gradient smoothness of the loss $\ell_c$ which is $2$ and $\mathcal L_c$ is the empirical loss defined in~\eqref{eq:lmp} with square loss substituted by $\ell_c$. Above, we used the fact that $\ell_c$ is bounded by $2c^2$.
\begin{figure}[H]
\centering
\includegraphics[width=0.7\textwidth]{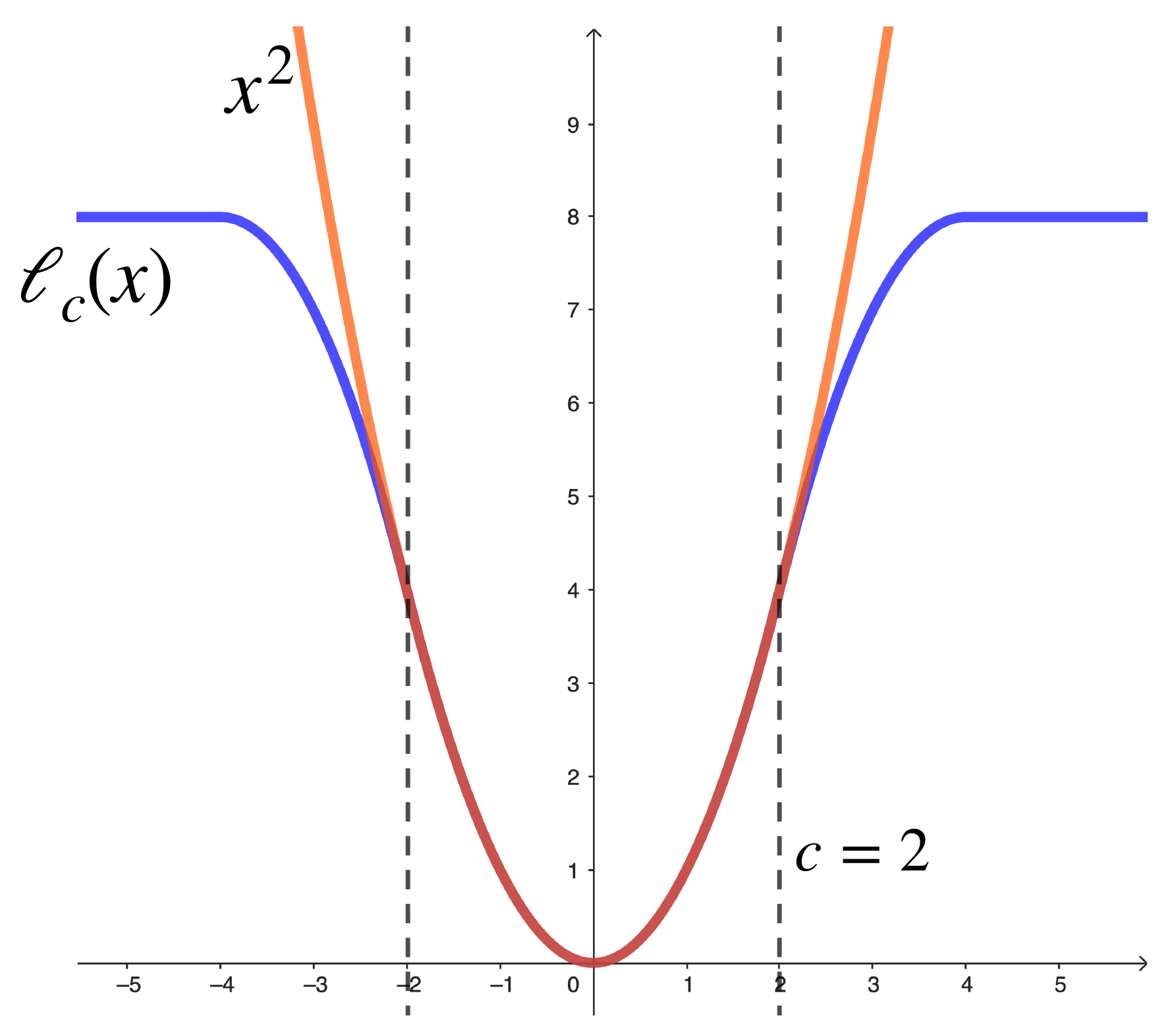}
\caption{The smooth surrogate loss $\ell_c$ as defined in Equation~\eqref{eq:truncated_smooth_l2_loss} with parameter $c=2$.}
\label{fig:surrogate}
\end{figure}
\end{proof}

\begin{lemma}\label[lemma]{lem:forthmomentbound}
    For standard Gaussian variable $X$, we have
    \begin{align*}
        \mathbb E \ind_{|X| \geq c} X^2 \leq e^{-c^2/2}\frac{2(c^2+2)}{c\sqrt{2\pi}}.
    \end{align*}
\end{lemma}

\begin{proof}[Proof of \Cref{lem:forthmomentbound}]
    \begin{align*}
         \mathbb E\ind_{|X| \geq c} X^2 = & 2/\sqrt{2\pi} \int_{x=c}^\infty x^2 e^{-\frac{x^2}{2}}dx \\
        \le & 2/\sqrt{2\pi} \int_{x=c}^\infty \frac{x^3}{c} e^{-\frac{x^2}{2}}dx \\
        = & 1/(c\sqrt{2\pi}) \int_{x=c}^\infty x^2 e^{-\frac{x^2}{2}}dx^2 \\
        = & 1/(c\sqrt{2\pi}) \int_{x=c^2}^\infty x e^{-\frac{x}{2}}dx \\
        = & 1/(c\sqrt{2\pi}) (- 0 - (-2 e^{-c^2/2}(c^2+2)))\\
        = & e^{-c^2/2}\frac{2(c^2+2)}{c\sqrt{2\pi}}.
    \end{align*}
\end{proof}

\subsection{Proof of \Cref{lem:RIPrelax}}
\begin{proof}[Proof of \Cref{lem:RIPrelax} ]
    Consider its SVD decomposition of $M$, $ M = \sum_{i=1}^d \alpha_i u_i v_i^T$, where $\alpha_i$'s are the singular values and $\{u_i\}_{i=1}^d,\{v_i\}_{i=1}^d$ each is an orthonormal basis for $\mathbb{R}^d$.  
    We can write
    \begin{align*}
       \sum_{i=1}^n \langle A_i, M \rangle^2&= \frac{1}{n}\sum_{i=1}^n (\sum_{j=1}^d \alpha_j u_j^T A_i v_j)^2\\
        &=\frac{1}{n}\sum_{i=1}^n \sum_{j:k=1}^d \alpha_j \alpha_k  \tr(A_i v_j u_j^T)\tr(A_i v_k u_k^T)\\
        &=\sum_{j:k=1}^d \frac{1}{4n}\sum_{i=1}^n \alpha_j \alpha_k  \big(\tr(A_i (v_j u_j^T + v_k u_k^T))^2 - \tr(A_i (v_ju_j^T - v_k u_k^T))^2\big).
    \end{align*}
    But again using the $(2,\delta)$-RIP of $\{A_i\}_{i=1}^n$, 
    \begin{align*}
        (1 - \delta)\|v_j u_j^T + v_k u_k^T\|_F^2 \leq\frac{1}{4n}\sum_{i=1}^n  \tr(A_i (v_j u_j^T + v_k u_k^T))^2 \leq (1 + \delta)\|v_j u_j^T + v_k u_k^T\|_F^2\\
        (1 - \delta)\|v_j u_j^T - v_k u_k^T\|_F^2 \leq  \frac{1}{4n} \sum_{i=1}^n \tr(A_i (v_j u_j^T - v_k u_k^T))^2 \leq (1 + \delta)\|v_j u_j^T - v_k u_k^T\|_F^2.
    \end{align*}
    This implies
    \begin{align*}
        &\frac{1}{4n}\sum_{i=1}^n \big(\tr(A_i (v_j u_j^T + v_k u_k^T))^2 - \tr(A_i (v_ju_j^T - v_k u_k^T))^2\big) \\
        &\leq
        \frac{1}{2}\delta(\|v_j u_j^T\|_F + \|v_k u_k^T\|_F) + (1+\delta)\langle v_ju_j^T, v_ku_k^T\rangle.
    \end{align*}
    Summing this over $j:k$ and noting that $\langle v_ju_j^T, v_ku_k^T\rangle$ is zero for $j \neq k$:
    \begin{align*}
         \sum_{i=1}^n \langle A_i, M \rangle^2 &\leq (1+\delta)(\sum_{j=1}^d \alpha_j^2) + \delta(\sum_j |\alpha_j|)^2
         \leq (1+\delta)\|M\|_F^2 + \delta \|M\|_*^2.\numberthis\label{eq:avalii}
    \end{align*}
    Similarly we obtain
    \begin{align}
     \sum_{i=1}^n \langle A_i, M \rangle^2 \geq (1-\delta) \|M\|_F^2 - \delta \|M\|_*^2.\label{eq:dovomi}
    \end{align}
    Combining Equations~\eqref{eq:avalii} and~\eqref{eq:dovomi}:
    \begin{align}
        \Big| \sum_{i=1}^n \langle A_i, M \rangle^2  - \|M\|_F^2 \Big|\leq  \delta\|M\|_F^2 + \delta\|M\|_*^2\le 2 \delta\|M\|_*^2.\label{eq:firsteq}
    \end{align}
This completes the proof.
\end{proof}


\section{Proof of \Cref{thm:lowerbound}}\label{sec:lowerboundproof}
\begin{proof}[Proof of \Cref{thm:lowerbound}]
Here we view matrices in $\mathbb R^{d_0\times d_L}$ as $d_0d_L$ dimensional vectors, hence by rotating a matrix with an orthogonal transformation we mean to rotate the corresponding vector.
Note that the minimum $\ell_2$ solution of the regression problem is given by $\widetilde M$ defined as
\begin{align*}
    \widetilde M = \sum_{i=1}^n A_i\Big[\Big(\langle A_i, A_j\rangle\Big)^{-1}_{1\leq i,j\leq n}b\Big]_i.
\end{align*}
First, note that if we rotate the ground-truth matrix $M^*$ with an arbitrary orthogonal matrix $U$, then $\widetilde M$ rotates according to the same $U$. Combining this with the fact the distribution on the measurement matrices is Gaussian and rotationally symmetric, we conclude that the population loss $\gL'(\widetilde M)$ is the same for all $M^*$. Hence, to lower bound the population loss, we can further assume that the entries of $M^*$ are sampled from standard Gaussian distribution. Hence, for any $M^*$ we can write
\begin{align*}
     \mathbb E_{\{A_i\}_{i=1}^n} \gL'(\widetilde M)
    &=\mathbb E_{M^*} \mathbb E_{\{A_i\}_{i=1}^n}\gL'(\widetilde M) \\
    &=  \mathbb E_{\{A_i\}_{i=1}^n}\mathbb E_{M^*} \|\widetilde M - M^*\|_F^2\\
    &= \mathbb E_{\{A_i\}_{i=1}^n}(1 - \frac{n}{d_0d_L})\|M^*\|_F^2\\
    &= (1 - \frac{n}{d_0d_L})\|M^*\|_F^2.
\end{align*}
where we used the fact that $\widetilde M$ is the projection of $M^*$ onto the subspace spanned by $\{A_i\}_{i=1}^n$.
\end{proof}


\end{document}